\documentclass{article}
\usepackage{microtype}
\usepackage{graphicx}
\usepackage{subfigure}
\usepackage{booktabs} 
\usepackage{tikz}
\usetikzlibrary{automata, positioning}
\usepackage{hyperref}

\usepackage[accepted]{MOSS_camera_ready}
\usepackage{amsmath}
\usepackage{amssymb}
\usepackage{mathtools}
\usepackage{amsthm}
\usepackage[capitalize,noabbrev]{cleveref}

\theoremstyle{plain}
\newtheorem{theorem}{Theorem}[section]
\newtheorem{proposition}[theorem]{Proposition}
\newtheorem{lemma}[theorem]{Lemma}

\theoremstyle{definition}
\newtheorem{definition}[theorem]{Definition}

\theoremstyle{remark}

\newcommand{\TC}{\texttt{TC}}

\usepackage{amsfonts}
\usepackage{amsmath}
\usepackage{amssymb}
\usepackage{mathtools}
\usepackage{mdframed}
\usepackage{amsthm}
\usepackage{thmtools, thm-restate}
\usepackage{dsfont}
\usepackage{mleftright}
\usepackage{standalone}

\newcommand{\R}{\mathbb{R}}

\newcommand{\st}{\;\middle|\;}

\newcommand{\set}[1]{\mleft\{#1\mright\}}

\usepackage[linecolor=black!20,textsize=tiny,disable]{todonotes}
\newcommand{\mehran}[1]
{\todo[backgroundcolor=blue!20]{\textbf{Mehran:} #1}}
\newcommand{\behnoush}[1]{\todo[backgroundcolor=orange!20]{\textbf{Behnoush:} #1}}

\icmltitlerunning{Parity Requires Unified Input Dependence and Negative Eigenvalues in SSMs} 

\begin{document}

\onecolumn
\icmltitle{Parity Requires Unified Input Dependence and Negative Eigenvalues in SSMs}
\icmlsetsymbol{equal}{*}

\begin{icmlauthorlist}
\icmlauthor{Behnoush Khavari}{equal,crl,yyy}
\icmlauthor{Mehran Shakerinava}{equal,yyy,mcguill}
\icmlauthor{Jayesh Khullar}{equal,xxx}
\icmlauthor{Jerry Huang}{equal,crl,yyy,uuu}
\icmlauthor {Fran\c{c}ois Rivest}{yyy,xxx,rmcc} 
\icmlauthor{Siamak Ravanbakhsh}{yyy,mcguill,fff}
\icmlauthor{Sarath Chandar}{crl,yyy,ppp,fff}
\end{icmlauthorlist}

\icmlaffiliation{crl}{Chandar Research Lab}
\icmlaffiliation{yyy}{Mila - Quebec AI Institute}
\icmlaffiliation{uuu}{Universit\'{e} de Montr\'{e}al}
\icmlaffiliation{xxx}{Queens University}
\icmlaffiliation{ppp}{Polytechnique Montr\'eal}
\icmlaffiliation{fff}{CIFAR AI Chair}
\icmlaffiliation{mcguill}{McGill University}
\icmlaffiliation{rmcc}{Royal Military College of Canada}

\icmlcorrespondingauthor{Mehran Shakerinava}{mehran.shakerinava@mail.mcgill.ca}
\icmlcorrespondingauthor{Behnoush Khavari}{khavarib@mila.quebec}
\icmlkeywords{Machine Learning, ICML}

\vskip 0.3in

\printAffiliationsAndNotice{\icmlEqualContribution}

\begin{abstract}
Recent work has shown that LRNN models such as S4D, Mamba, and DeltaNet lack state-tracking capability due to either time-invariant transition matrices or restricted eigenvalue ranges. To address this, input-dependent transition matrices, particularly those that are complex or non-triangular, have been proposed to enhance SSM performance on such tasks. While existing theorems demonstrate that both input-independent and non-negative SSMs are incapable of solving simple state-tracking tasks like parity, regardless of depth, they do not explore whether combining these two types in a multilayer SSM could help. We investigate this question for efficient SSMs with diagonal transition matrices and show that such combinations still fail to solve parity. This implies that a recurrence layer must both be input-dependent and include negative eigenvalues. Our experiments support this conclusion by analyzing an SSM model that combines S4D and Mamba layers.
\end{abstract}

\section{Introduction}

While efficiency and compute have motivated alternatives to Transformers, an equally important concern is understanding the failure modes of different architectures. Addressing these failures is key to designing better models and requires analyzing three aspects: (1) the model’s intrinsic expressive capacity, (2) whether learning dynamics (e.g., gradient descent) can reach solutions within that capacity, and (3) the practical impact of these limitations on real tasks. In this work, we focus on the first aspect, architectural expressivity, and examine structural constraints in state space models (SSMs) that cause them to fail on simple synthetic tasks like parity. Expressivity is crucial because a model may perform well on a particular distribution, but it will fail to generalize to out-of-distribution (OOD) inputs if it cannot represent the correct underlying algorithm.

Examples of Transformers~\cite{liu2023transformers} and linear RNNs~\cite{sarrof2024expressive} failing to generalize to out-of-distribution (OOD) inputs highlight how limited expressivity can lead models to rely on shortcut solutions that do not generalize beyond the training distribution. These failures may stem from practical constraints such as finite-precision computation or other factors.\behnoush{a bit vague. This other factors can be specified, like saying that non-linearities should learn operations like mod, which can break beyond the values that model learned during training.} A particularly illustrative class of tasks where such models are known to fail is state-tracking~\cite{deletang2023neural,hahn2024sensitive,bhattamishra2022simplicity,merrill2024illusion,sarrof2024expressive,grazzi2024unlocking}, a subset of regular languages in formal language theory that includes simple tasks like parity. State tracking tasks are considered representative of a model’s performance on real-world problems, such as code execution and program analysis.
\mehran{I don't think you should ever use line breaks like this (``\textbackslash\textbackslash"). Just start a new paragraph.}

Recent work has examined the failure modes of linear recurrent models, including SSMs and linear Transformers, on state-tracking tasks. Starting with a study by \citet{merrill2024illusion} that highlights this issue in the context of tracking states over sequences of non-solvable group operations, later work by~\citet{sarrof2024expressive,grazzi2024unlocking} reveals that these models are unable to solve even solvable group operation tasks such as parity, due to design limitations in their state transition matrices; that is, these matrices either lack input dependence or have no negative (or complex) eigenvalues, both of which are essential for solving state-tracking tasks. We defer a detailed discussion of these findings to Section~\ref{section:related-work}.\mehran{Reference Section 2.}
\behnoush{Done!
}
Our goal in this paper is to test whether these architectural limitations can be circumvented by combining desirable properties, namely input dependence and the use of negative or complex eigenvalues, in a multi-layer setup. Interestingly, we find that merely layering SSMs with complementary properties (e.g., stacking Mamba\mehran{real/positive Mamba. It might be good to specify this once in the beginning.} and S4D) does not overcome these limitations. Instead, we show, both theoretically and empirically, that to solve even a simple task like parity, a single recurrence layer must simultaneously satisfy both properties.

Our contributions are as follows:\\
1.
We prove that layer-wise combinations of non-negative input-dependent and negative input-invariant SSMs (e.g., Mamba and S4D) cannot solve parity in finite precision, even when skip connections are allowed.\\
2.
We introduce new failure modes for non-negative, input-dependent SSMs and show how they persist across various SSMs.
\\
3.
We provide a constructive example demonstrating how S4D with complex eigenvalues can solve modular counting, contrasting it with the limitations of Mamba~\cite{grazzi2024unlocking}.
\\
4.
We conduct empirical experiments to validate the theory and demonstrate the inability of hybrid models consisting of S4D and Mamba to extrapolate beyond training lengths on the parity task.
\section{Background}
We briefly review structured state space models (SSMs) and provide background on two relevant variants: S4D and Mamba. State space models describe the relationship between an input signal $x(t)$ and output $y(t)$ through a hidden state $h(t)$ using the equations
$
    \dot{h}(t) = Ah(t) + Bx(t),
    y(t) = Ch(t) + Dx(t),
$
where $x(t) \in \R$, $y(t) \in \R$, $A \in \R^{N \times N}$, $B \in \R^{N \times 1}$, $C \in \R^{1 \times N}$, and $D \in \R$.

A discretized version of this model was introduced to the deep learning community by~\citet{gu2021efficiently}, with recurrence and output equations:
$
    h_t = Ah_{t-1} + Bx_t,
    y_t = \sigma(Ch_t + Dx_t),
$
where $\sigma$ denotes an output nonlinearity. Careful initialization of the parameters, especially the transition matrix $A$ (e.g., via HiPPO~\cite{gu2020hippo}), allows these models to mitigate the vanishing gradient problem that affects classical RNNs. Moreover, the structure imposed on the $A$ matrix (normal plus low-rank) makes learning efficient. These innovations led to the Structured State Space Sequential model (S4)~\cite{gu2021efficiently}, which achieved state-of-the-art results on a range of long-range sequence modeling tasks~\cite{tay2021long}, where transformers had previously struggled. As a result, S4 was seen as a promising alternative or complement to attention-based models.

S4 inspired several follow-up models. On the one hand, simpler variants such as DSS~\cite{gupta2022diagonal}, S4D~\cite{gu2022parameterization}, and S5~\cite{smith2023simplified} simplified S4's architecture while retaining strong performance. On the other hand, more sophisticated models such as H3~\cite{fu2022hungry} and Mamba~\cite{gu2024mamba} extended SSMs to handle a more diverse set of tasks, particularly language modeling. We now briefly describe S4D and Mamba, two models we focus on in this paper, and highlight how they differ from the original S4.

Note that since SSMs relate the states $x_t$ and $x_{t-1}$ via a linear equation, they are also known as linear recurrent neural networks (LRNNs). We use ``SSM'' and ``LRNN'' interchangeably.

\vspace{-1em}
\paragraph{\textbf{S4D}} S4D is a simplified version of S4 in which the complex transition matrix $A$ is constrained to be diagonal. This reduces the cost of matrix multiplication in the recurrence equations and hence leads to more efficient computation.

\vspace{-1em}
\paragraph{\textbf{Mamba}} Designed specifically for language modeling tasks, Mamba introduces input dependence (also called selectivity) into the state and output equations:
$
    h_t = A(x_t) h_{t-1} + B(x_t)x_t,
    y_t = \sigma(C(x_t)h_t + D(x_t)x_t).
$
Unlike time-invariant models like S4D, Mamba’s parameters vary with the input, making it input-dependent or time-variant. A critical design choice in Mamba, which, as we will see later, significantly affects its performance on state-tracking tasks, is that the transition matrix $A$ is real and diagonal. Furthermore, the original implementation (commonly used in practice) restricts $A$ to have only non-negative entries. In the next section, we will see that this first constraint prevents Mamba from solving many state-tracking tasks, such as modular counting. The further restriction to non-negative eigenvalues makes it incapable of solving even the parity task.

\section{Related Work}~\label{section:related-work}
\vspace{-2em}

\paragraph{\textbf{Difficulty of state-tracking for SSMs and Linear RNNs}} \citet{merrill2024illusion} argue that the linear recurrence in most (parallelizable) SSMs places them in the complexity class $\TC^0$, the class of problems solvable by constant-depth, polynomial-size threshold circuits. This class is widely believed to be incapable of expressing non-solvable state-tracking tasks such as $S_5$\footnote{$S_5$ is the symmetric group on five elements, representing all permutations of five distinct objects.}. This implies that parallelizable SSMs cannot solve $S_5$ or similarly complex state-tracking tasks for long sequences unless the model depth scales with sequence length. To address this, \citet{merrill2024illusion} propose either adding nonlinearities to the recurrence or making the recurrence input-dependent. They show empirically that the latter approach enables a single-layer S4D to solve non-solvable state-tracking tasks across arbitrary sequence lengths.

\paragraph{\textbf{Difficulty of Parity}} While~\citet{merrill2024illusion} focus on the limitations of linear RNNs compared to full RNNs on non-solvable state-tracking tasks, \citet{sarrof2024expressive} highlight a significant gap between linear and nonlinear RNNs, even on solvable tasks such as parity, though for different reasons. They prove that multilayer input-dependent linear RNNs (SSMs) with diagonal, non-negative transition matrices (e.g., Mamba) cannot solve parity in finite precision for arbitrary sequence lengths (their Theorem 2), despite being able to solve all star-free regular tasks. They also show that time-invariant diagonal SSMs with negative eigenvalues (e.g., S4D) fail on parity due to the lack of input dependence (Theorem 13).

Extending these results to non-diagonal models, \citet{grazzi2024unlocking} prove that a multilayer SSM can solve parity in finite precision for arbitrary sequence lengths only if at least one layer has a negative eigenvalue. This implies that even DeltaNet\footnote{DeltaNet is a linear attention model that can also be interpreted as an SSM, with a diagonal plus low-rank transition matrix. More specifically, a generalized Householder matrix.} fails to solve parity in its standard form. They argue that existing SSMs typically lack either input dependence or negative (or complex) eigenvalues. Both of these are essential for solving parity and, more generally, for complex state-tracking tasks. To address this, they modify Mamba and DeltaNet to allow eigenvalues in the range $[-1, 1]$ instead of $[0, 1]$. This leads to empirical improvements on both parity and more challenging tasks like modular arithmetic.

While negative eigenvalues are sufficient for solving parity, \citet{grazzi2024unlocking} show that complex eigenvalues are necessary for harder tasks such as modular counting.\footnote{In a way, the modular counting task is simpler than parity, since the input is fixed to $1$, hence input dependence is not required.} Thus, although modifying Mamba to include negative eigenvalues enables it to solve parity, solving more complex tasks demands transition matrices with complex eigenvalues. They note that such matrices can be constructed by multiplying several real-eigenvalued matrices, provided the product is not triangular.\footnote{Because the eigenvalues of a triangular matrix are the diagonal elements, and the product of two triangular matrices remains triangular.}

Building on this idea, \citet{siems2025deltaproduct} propose DeltaProduct, an adaptive extension of DeltaNet that generalizes the transition matrix from diagonal-plus-rank-1 to a structured rank-$n$ matrix. Here $n$ is tunable to trade off between expressivity and efficiency. Their construction is based on products of $n$ generalized Householder matrices. A limitation of this approach is the computational cost of multiplying non-diagonal matrices. Moreover, Theorem 1 in~\citet{grazzi2024unlocking} does not specify whether the single layer containing negative or complex eigenvalues must also be input-dependent. It is unclear whether input dependence and eigenvalue complexity can be separated across different layers. In that case, a layer-wise combination of the two diagonal models of Mamba and S4D may provide a more efficient solution to the problem. We address this question in our main theorem in the next section.

\section{Theory}
We begin by reviewing Theorem 2 from \citet{sarrof2024expressive} and Theorem 1 from \citet{grazzi2024unlocking}, which establish the impossibility of solving parity with multilayer SSMs that do not simultaneously satisfy input-dependence and negativity conditions. The proofs in both works for the non-negative case rely on the failure of such models on a specific input, a sufficiently long sequence of all ones. Next, we prove that a one-layer, input-invariant model like S4D, with complex eigenvalues, can solve modular counting for any modulus, and accordingly can solve parity on that specific input. This result implies that the failure mode underpinning the negative result for non-negative SSMs on parity can be bypassed by adding a single S4D layer. Intuitively, because the input is constant, input dependence is not required for this task. Consequently, we propose a different class of input sequences that not only defeat non-negative SSMs but also challenge time-invariant SSMs. We show that even a layer-wise combination of positive and time-invariant SSMs fails to model parity on these sequences.
\subsection{non-negative SSMs Cannot Solve Parity}
Parity is the language of bitstrings with an even number of ones.
\begin{theorem}[Theorem 1 of \citet{grazzi2024unlocking}]\label{thm:parity}
A finite precision LRNN with finitely many layers can solve parity for arbitrary input lengths only if in at least one layer, there exists $x$ such that $A(x)$ has at least one eigenvalue $\lambda \notin \set{x \in \R \st x \geq 0}$.
\end{theorem}
In~\cref{app:proof-unlocking}, we provide a simplified overview of their proof in the diagonal case. The key point is that the argument relies on the specific input sequence 
$1^T$, and the crux of the proof is that the SSM’s hidden state converges to a fixed value after sufficiently many steps on this input.  In the following section, we prove that adding a time-invariant SSM layer with negative eigenvalues enables the model to solve parity on $1^T$, meaning this convergence-based argument no longer applies to the combination of positive Mamba and S4D. Building on this, the next section introduces a distinct failure mode.

\subsection{Modular Counting}
Modular counting  is the problem of counting modulo $m$ with the input of length $T$ being  $1^T$ and the desired output at the final step is $T \bmod m$.
The following theorem states that time-invariant SSMs with complex eigenvalues can solve this task.
\begin{proposition}\label{thm:modular-counting}
A single-layer S4D can solve modular counting.
\end{proposition}  

\begin{proof}
Since the input sequence is always the same, i.e., $1$'s, we expect that there exist a solution independent of the weight $B$; keeping the input and hidden space dimension as one, we let $B = 0$, $A = \exp(2\pi i/T)$ and $h_0 = 1$. If we take $h_0$ to be the accept state then this S4D model recognizes the language $(1^T)^*$.~\behnoush{why do we need * here?}
\end{proof}
Since solving parity on the sequence $1^T$ is equivalent to doing modular counting with modulus $2$ on it
, the above theorem implies that the failure mode used to prove~\cref{thm:parity} can be circumvented by adding a time-invariant SSM layer including complex (even only negative ) eigenvalues. This raises the question of whether parity task can be solved by a combination of input-dependent and complex-valued recurrence layers. To answer this question, we suggest another failure mode for non-negative SSMs that cannot be alleviated by adding a time-invariant layer, even with complex or negative eigenvalues.

\subsection{Another Failure Mode for  non-negative SSMs}
\begin{lemma}
Non-negative diagonal SSMs fail to solve parity task for arbitrary sequence length in finite precision on any input sequence of the form 
\(\,
\underbrace{
  \underbrace{0 \ldots 0}_{k\ \text{zeros}} 
  \quad 
  \underbrace{1 \ldots 1}_{m\ \text{ones, } m\ \text{odd}}
}_{\text{one cycle}} 
\quad
\underbrace{0 \ldots 0}_{k\ \text{zeros}} 
\quad 
\underbrace{1 \ldots 1}_{m\ \text{ones, } m\ \text{odd}} 
\cdots\,
\).
\end{lemma}
\begin{proof}
    Starting with $k=m=1$, with the input sequence simplified to $(01)^T$, the first two steps of the recurrence~\eqref{eq:ssm} give
    \begin{align}
        h_1 = A(0)h_0 + B(0)\,,\quad
        h_2 = A(1)A(0)h_0 +A(1)B(0) + B(1) 
    \end{align}
    By defining $A(01) \coloneqq A(1)A(0)$ and $B(01) \coloneqq A(1)B(0) + B(1)$ 
    and starting from $h_0$, for any even step $2t, t\in\mathbb{Z}$, we have
    $h_{2t} = A(01)h_{2t-2} + B(01)$.
For $A(0)$ and $A(1)$ positive semi-definite (PSD), which is the case for all SSMs we are aware of, $A(01)$ is guaranteed to be non-negative~(proof in~\cref{app:psd-product}). Hence, this ``two-step'' recurrence will have the same dynamics as the one in~\cref{eq:ssm} with non-negative  transition matrices. In particular, the states at even steps in the sequence, follow a similar evolution to the one described by ~\cref{eqn:state-evolution} in the proof of~\cref{thm:parity}. The only difference is that $A(1)$ and $B(1)$ are replaced by $A(01)$ and $B(01)$. Therefore, the finite precision assumption again results in the state collapse after some threshold $\tau_0$. As before, this means that non-negative SSMs cannot model parity on the sequence $(01)^T$, since the state after each length-$2$ cycle of the form $01$ should flip, while its collapse means that the state remains constant. The same logic goes for any sequence of the form used in the statement of the lemma, i.e., with an odd number of ones in each cycle. The key point is that the product of any number of non-negative diagonal matrices remains non-negative.
\end{proof}

\subsection{Can S4D and Mamba Together Solve Parity?} The following theorem states that combining S4D and Mamba in different layers still cannot solve parity. This result can be generalized to the case of a model made of non-negative input-dependent SSM layers and time-invariant SSM layers with negative eigenvalues, where both types of layers are diagonal.The proof is provided in~\cref{app:proof-main-theorem}.

\begin{theorem}\label{thm:mamba_and_s4d}
A finite precision LRNN consisting of S4D and Mamba layers and skip connections and learnable initial state cannot solve parity.
\end{theorem}

\section{Experiments}
The code for our experiments is available at \href{https://github.com/abhishekpanigrahi1996/MOSS/tree/main/submissions/submission-92}{this repository}. Our implementation for RNN is based on the PyTorch RNN layer. For S4D, we use the original implementation from \href{https://github.com/state-spaces/s4/blob/main/models/s4/s4.py}{state-spaces/s4}. For Mamba, we use the implementation from \href{https://github.com/johnma2006/mamba-minimal}{mamba-minimal}.

\paragraph{Parity} We try the parity task with four types of models: vanilla RNN, two-layer S4D, two-layer Mamba, two-layer model with both S4D and Mamba (Model sizes are given in~\cref{app:exp-parity}). We show the summary of results for training the models on sequences of length $8$ and their extrapolation to longer length up to $10K$. The main observation is that extrapolation remains an issue for all models except RNN. The training accuracy of SSMs indicates non-trivial performance on the train set, whereas performance on the longer sequences remains around 50\%, which means that linear models did not learn the correct algorithm for parity.
\begin{table}[h]
    \centering
    \begin{tabular}{c|cc}
    \toprule
    Model & Train Accuracy & Extrapolation Accuracy \\
    \midrule
    RNN & 100\% & 100\% \\
    S4D & 99.7\% & 50\% \\
    Mamba & 100\% & 50\% \\
    Mamba + S4D & 100\% & 50\% \\
    \bottomrule
    \end{tabular}
    \caption{Performance of various models on the parity task.}
    \label{tab:parity}
\end{table}
\paragraph{Offset Prediction} We design a counting task that is very similar to modular counting. Here the goal is to examine the practical implications of~\cref{thm:modular-counting} on a real task. While modular counting by definition has a fixed input, this variation allows the model to see different inputs. The input is a binary on/off signal: a stimulus of 1 is presented for a fixed 10-timestep interval (inter-stimulus interval, or ISI), followed by a stimulus of 0 for a randomly selected duration from $[20, 40]$ (inter-trial interval, or ITI). This ITI/ISI cycle repeats throughout a sequence of fixed length $200$ timesteps. Each sequence begins with an ITI, followed by alternating ITIs and ISIs. The task requires the models to detect an ISI onset to count till at the end of the stimulus, encouraging them to learn the fixed duration and representation of the ISI. In ~\cref{app:exp-modular-counting}, we propose an analytic solution for it that S4D can find. Interestingly, as shown in~\cref{fig:offset}, both S4D and Mamba can solve this task. For Mamba, we still do not have an explanation how it can solve the task. ~\cref{app:exp-modular-counting} provides all the details on the theory and results for this task.
\begin{figure}[H]\label{fig:offset}
    \centering
    \includegraphics[width=0.7\linewidth]{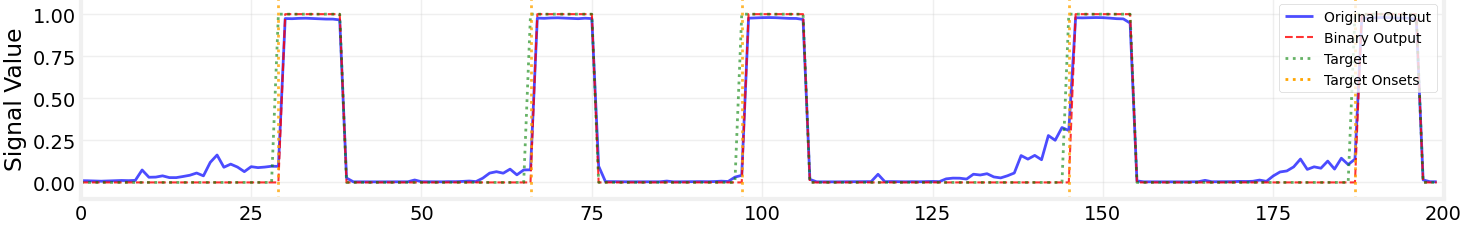}
    \includegraphics[width=0.7\linewidth]{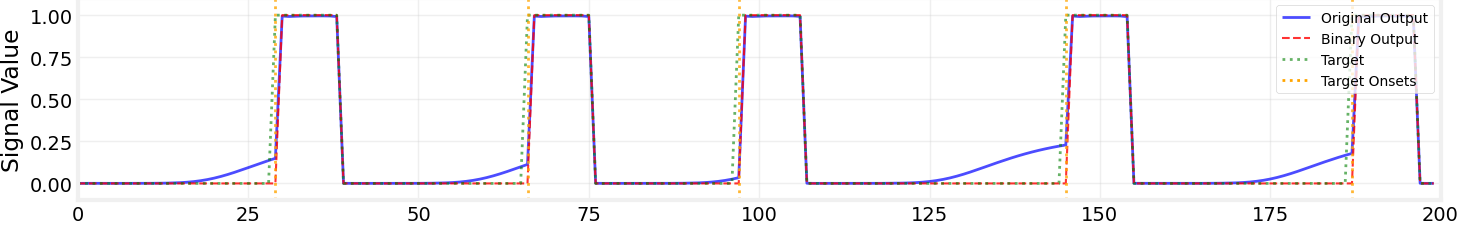}
    \caption{(\textbf{Top}) Output of S4D on the offset prediction task, trained for 1k epochs with a hidden size of 8 and a learning rate of 0.005. (\textbf{Bottom}) Output of Mamba on the offset prediction task, trained for 1k epochs with a hidden size of 8 and a learning rate of 0.01. Both models predict the offset correctly and anticipate the onset.}
    \label{fig:offset}
\end{figure} 

\section{Conclusion, Limitations, and Future Directions}
The core question we addressed in this paper is whether a layer-wise combination of input-independent negative SSMs (like S4D) and input-dependent non-negative ones (such as Mamba) can solve parity on all sequences, or both properties should be present in a single recurrence layer simultaneously. Our theoretical result shows that for the diagonal case, combining these two types of recurrence in different layers does not enable the model to solve parity. Hence, we hypothesize that there should exist at least one recurrence layer with both properties. Possible future directions include generalizing our results to non-diagonal recurrence for the parity task and to the complex-valued recurrence for more complicated state tracking tasks.

\paragraph{Limitations} Our study simplifies the SSM formulation by using a vector-valued hidden state, as in~\citep{sarrof2024expressive}, instead of the common matrix representation used in SSMs. While we do not expect this to affect our results, a careful analysis using the full formulation is a valuable direction for future work. Additionally, accounting for other architectural components, such as layer normalization or dropout, would further strengthen the analysis.
\section*{Acknowledgements}
We thank Olexa Bilaniuk for assistance with Appendix E in the paper. We also thank Andreas Madsen for valuable discussions on related topics and for developing the excellent \emph{lrcurve} package which was used for real-time visualization in this submission.
Sarath Chandar was supported by the Canada CIFAR AI Chairs program, the Canada Research Chair in Lifelong Machine Learning, and the NSERC Discovery Grant. Siamak Ravanbakhsh was in part supported by Canada CIFAR AI Chairs and NSERC Discovery
programs. Fran\c{c}ois Rivest was supported by the NSERC discovery development grant and CDARP.
\bibliography{references}
\bibliographystyle{icml2025}

\newpage

\appendix
\section{Non-negative SSMs Cannot Solve Parity}\label
{app:proof-unlocking}
For simplicity, we focus on the diagonal SSM 
 and adjust their proof to it. Consider the SSM equations as
\begin{align}\label{eq:ssm}
h_t = A(x_t) h_{t-1} + B(x_t)\,, \quad 
y_t = \sigma(h_t, x_t)
\end{align}
where $\sigma$ represents a general non-linear function of $h_t$ and $x_t$. By unrolling the recursion relation, we get the hidden state at step $k$ given by
$
h_k = \prod _{i=1}^k A(x_i) h_0  + \sum_{i=1}^k B(x_i)\prod_{j=i+1}^k A(x_j),
$
with the convention of $\prod_{j=k+1}^k A(x_j) = I$ to reduce clutter. 
Now, if we assume the input sequence is $x_1\dots x_k = 1^k$, the state at time $k$ is
\begin{equation}\label{eqn:state-evolution}
    h_k = A(1)^k h_0 + \sum_{i=1}^k B(1) A(1)^{k-i} = A(1)^k h_0 + \sum_{i=0}^{k-1} B(1) A(1)^{i}
\end{equation}
Next step is to show that in finite precision, this value converges to a constant after some sequence step $\tau_{\text{cut-off}}$.
Since the powers of   a diagonal matrix are given by the matrix of the diagonal elements, $\lambda_i$, each to that same power, and here the matrix $A$ is positive, the elements of $A(1)^k$ and $A(1)^i$ after some given time $t_0$ will either grow to beyond the defined precision (if $\lambda_i>1$), or fall below it and converge to zero (if $\lambda_i<1$). For $\lambda_i = 1$, if $B(1)=0$, it converges to some fixed values, otherwise the sum term continues growing until it goes beyond the fixed precision and hence converges to a fixed rounded value.
The proof for non-diagonal case is given in~\cite{grazzi2024unlocking}.


\section{Proof of \cref{thm:mamba_and_s4d}}\label{app:proof-main-theorem}

\begin{proof}
For every such model, we construct an input sequence on which it fails to produce the correct output. We follow the proof structure of Theorem 13 from \citet{sarrof2024expressive}, while extending it to account for the possibility of positive Mamba layers, skip connections, and arbitrary initial states. Note that the $A$ matrix in both S4D and Mamba is diagonal. Moreover, in Mamba, the diagonal entries are non-negative.

In any S4D layer, each diagonal entry $A_j \in \mathbb{C}$ can be written as $A_j = r_j \exp(2\pi i q_j)$, where $q_j \in [0,1]$ is rational and $r_j \in \mathbb{R}_{\geq 0}$. Since there are finitely many S4D layers and each layer has finitely many such $A_j$, we can choose a positive integer $W$ such that $W q_j \in \mathbb{N}$ for every $j$ in each layer. Importantly, this ensures $(A_j)^W = (r_j)^W \in \mathbb{R}_{\geq 0}$.

Let $s = 10^{W-1}$ (i.e., a string consisting of a single one followed by $W - 1$ zeros). We will show that the model fails to compute parity on the input $s^T$ (the string $s$ repeated $T$ times) when $T$ is sufficiently large.

Consider the state dynamics of any layer of the model (whether S4D or Mamba) on an input sequence of the form $(x)_i = (x_1 \dots x_W)^*$\footnote{$*$ denotes the Kleene star: zero or more repetitions.}. Note that $s^T$ is of this form. The claim is that for each layer $k$ and for each $i = 1, \dots, W$, the sequence $(h^{(k)}_{tW+i})_{t \in \mathbb{N}}$ converges in the finite-precision setting as $t \rightarrow \infty$. In this context, convergence means that the sequence becomes stationary after some finite time. Since the parity of $s^T$ equals the parity of $T$, it becomes impossible to extract the parity from $h_{TW}$ for large $T$.

We prove the claim by induction on the number of layers. For simplicity, we treat the state of the zeroth layer as the raw input. The base case holds since the input $s^T$ is of the form $(x_1 \dots x_W)^*$.

Assume the input to the $k$-th layer is of this form and show that its output will also be of the same form. We suppress the layer index $k$ and dimension index for clarity. This layer is either an S4D layer or a positive Mamba layer.

\textbf{Case 1: S4D layer.} Consider the sequence $(h_{tW})_{t \in \mathbb{N}}$ and let $h_0$ be the state of the system after the input has stabilized. The derivation, similar to \citet{sarrof2024expressive}, proceeds as follows:
\begin{align*}
    h_{tW} &= A^{tW}h_0 + \sum_{k=1}^{tW} A^{tW-k} B(x_k) \\
           &= r^{tW} h_0 + \sum_{j=1}^{W} \sum_{k=0}^{t-1} A^{(t-k)W-j} B(x_{kW+j}) \\
           &= r^{tW} h_0 + \sum_{j=1}^{W} \sum_{k=0}^{t-1} r^{(t-k)W} A^{-j} B(x_j) \\
           &= r^{tW} h_0 + \sum_{j=1}^{W} \left( \sum_{k=0}^{t-1} r^{(t-k)W} \right) \left( A^{-j} B(x_j) \right) \\
           &= \underbrace{r^{tW} h_0}_{U_1} + \sum_{j=1}^{W} \underbrace{\left( \sum_{k=1}^{t} r^{kW} \right)}_{U_2} \underbrace{\left( r^{-j} \exp(-2\pi i j q) B(x_j) \right)}_{U_3}
\end{align*}

Here, $U_2$ is independent of $t$. Intuitively, $U_3 \in \mathbb{C}$ determines a direction in the complex plane, while $U_2 \in \mathbb{R}$ determines the magnitude. $U_1$ may converge or diverge exponentially. Similarly, $U_2$ may converge exponentially, diverge linearly, or diverge exponentially. Consequently, in the finite-precision setting, $h_{tW}$ must converge for sufficiently large $T$.

\textbf{Case 2: Positive Mamba layer.} The idea here is to approximate linear time-invariant (LTI) behavior at a coarse-grained level. Again, consider the sequence $(h_{tW})_{t \in \mathbb{N}}$ and let $h_0$ be the state of the system after its input has stabilized. Let $\tilde{A} = A(x_1) \dots A(x_W)$. Then:
\begin{align*}
    h_{tW} &= \left( A(x_1) \dots A(x_{tW}) \right) h_0 + \sum_{k=1}^{tW} \left( A(x_{k+1}) \dots A(x_{tW}) \right) B(x_k) \\
           &= \tilde{A}^{t} h_0 + \sum_{j=1}^{W} \sum_{k=0}^{t-1} \left( A(x_{kW+j+1}) \dots A(x_{tW}) \right) B(x_{kW+j}) \\
           &= \tilde{A}^{t} h_0 + \sum_{j=1}^{W} \sum_{k=0}^{t-1} \tilde{A}^{(t-k)} \left( A(x_{j+1}) \dots A(x_W) \right) B(x_j) \\
           &= \underbrace{\tilde{A}^{t} h_0}_{U_1} + \sum_{j=1}^{W} \underbrace{\left( \sum_{k=1}^{t} \tilde{A}^k \right)}_{U_2} \underbrace{\left( A(x_{j+1}) \dots A(x_W) B(x_j) \right)}_{U_3}
\end{align*}

As before, $U_3$ does not depend on $t$. Both $\tilde{A}^t$ and $U_2$ are diagonal matrices with non-negative entries. Thus, $U_1$ can converge or diverge exponentially, and $U_2$ can converge exponentially, diverge linearly, or diverge exponentially. Hence, in the finite-precision setting, $h_{tW}$ must converge for sufficiently large $T$.

Since both the input and state fall into a cycle of length $W$, any non-recurrent function of their concatenation also falls into a cycle of length $W$. This completes the induction step and proves our claim for any finite number of layers. Allowing for concatenation shows that skip connections do not increase expressivity in this setting. The proof also holds regardless of the system’s initial state.
\end{proof}

\section{Offset Prediction Task}\label{app:exp-modular-counting}
Here, we show an analytic solution for S4D for this task. Next, we present our results.
\subsection*{A Solution for Offset Prediction Task with S4D}

\mehran{Unify the name of the task: it is called triggered fixed counting, offset prediction, and counting... Use a single name consistently.}
We consider the task of next-token prediction for $1$s and random bursts of $0$s of fixed length $n$. We show that one layer of S4D with complex scalar state is capable of correctly predicting all tokens except for the first token of a burst. At a high-level, we will be simulating a finite-state automaton that remains in a \textit{sleep} state when observing $1$s, and \textit{counts} to $n$ when observing $0$s, ending up back at the sleep state. There are thus $n$ states in total. Note that the automata will not be able to handle interruptions to the counting procedure, so it will go into a fail state if the number of $0$s is not a multiple of $n$. The model will predict $1$ at the sleep state and $0$ otherwise. Letting $h_0 = 1, A = \exp(i  2\pi / n), B = 1 - \exp(i  2\pi /n)$ constructs the desired system. The sleep state is $h = 1$. The output function is $\phi(h_t) = \mathbf{1}_{\{h_t = 1\}}$.

\begin{figure}[!h]
  \centering
  \resizebox{0.6\linewidth}{!}{%
    \begin{tikzpicture}[font=\large, ->, >=stealth, shorten >=1pt, node distance=2.8cm, on grid, auto]
      \node[state, initial] (s0) {$s_0$};
      \node[state] (s1) [right=of s0] {$s_1$};
      \node[state] (s2) [right=of s1] {$s_2$};
      \node[state] (s3) [right=of s2] {$s_3$};
      \node[state] (fail) [below=2cm of s2] {fail};

      \path (s0) edge[loop above] node {$1$} (s0);
      \path (s0) edge node {$0$} (s1);
      \path (s1) edge node {$0$} (s2);
      \path (s2) edge node {$0$} (s3);
      \path (s3) edge[bend right=25] node[above] {$0$} (s0);

      \path (s1) edge[bend left=20] node[below] {$1$} (fail);
      \path (s2) edge node[right] {$1$} (fail);
      \path (s3) edge[bend right=20] node[below] {$1$} (fail);

      \path (fail) edge[loop right] node {$0,1$} (fail);
    \end{tikzpicture}
  }
  \caption{Finite-state automaton for offset prediction task when $n = 4$.}
  \label{fig:fsa-n4}
\end{figure}
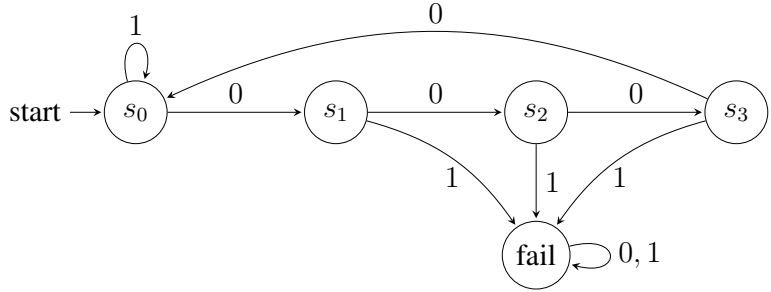
\subsection{Results}
In this section, we provide the results for the offset prediction task and generalization performance of S4D and Mamba on the offset prediction task. The two setups designed for evaluating their generalization are explained below.

\begin{figure}[H]
    \centering
    \includegraphics[width=0.8\linewidth]{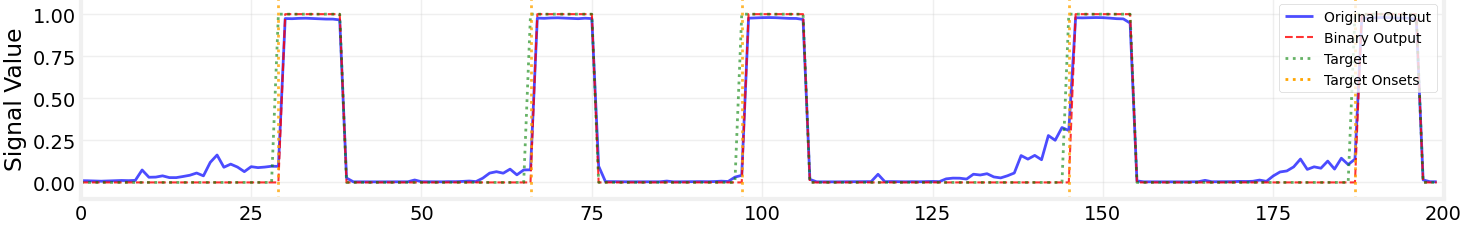}
    \caption{Offset prediction of S4D. We see that S4D correctly predicts the offset for all ISIs.}
    \label{fig:two-isi}
\end{figure}

\begin{figure}[H]
    \centering
    \includegraphics[width=0.8\linewidth]{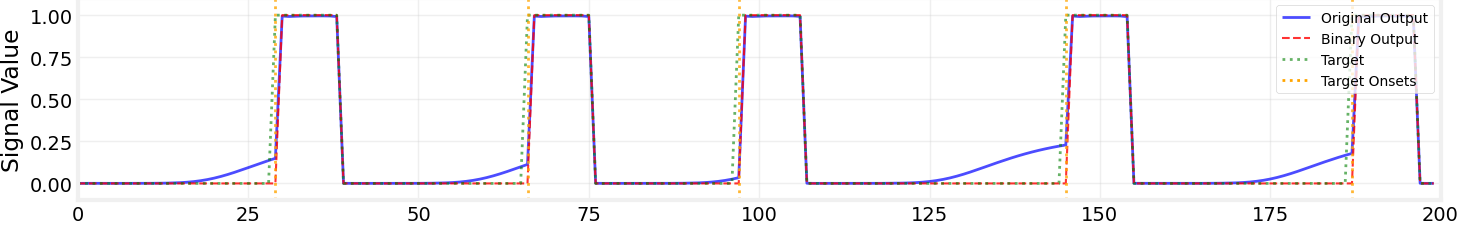}
    \caption{Offset prediction of Mamba. We see that Mamba correctly predicts the offset for all ISIs.}
    \label{fig:two-isi}
\end{figure}

\begin{figure}[H]
    \centering
    \includegraphics[width=0.8\linewidth]{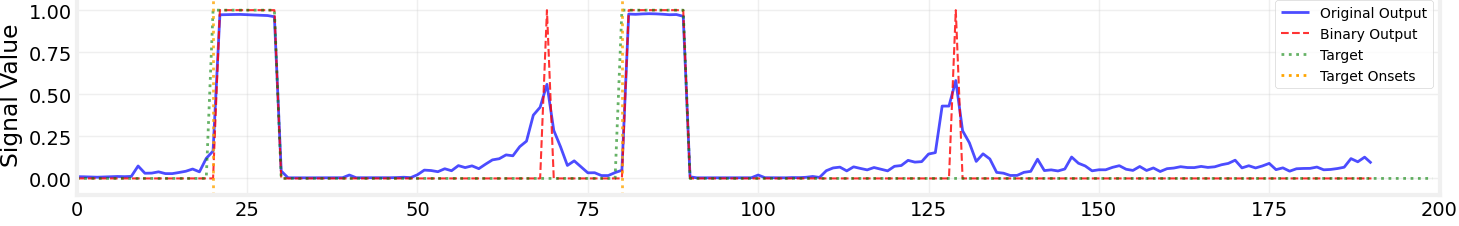}
    \caption{Two ISI outputs of S4D. We see that S4D correctly predicts the offset for both ISIs.}
    \label{fig:two-isi}
\end{figure}

\begin{figure}[H]
    \centering
    \includegraphics[width=0.8\linewidth]{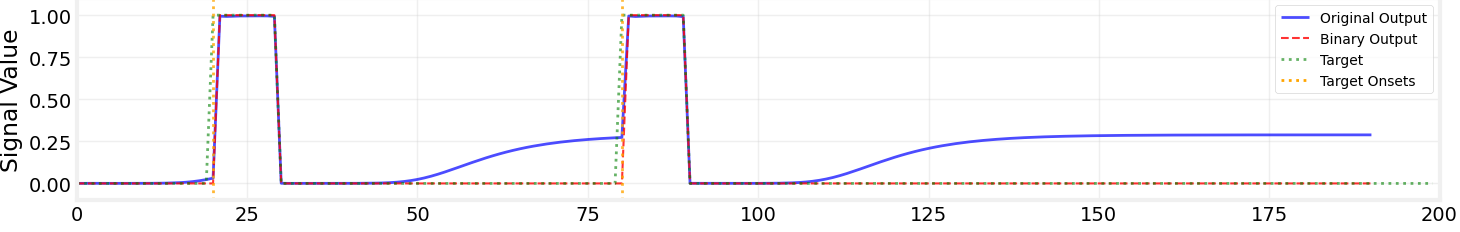}
    \caption{Two ISI outputs of Mamba. We see that Mamba correctly predicts the offset for both ISIs.}
    \label{fig:two-isi}
\end{figure}

\paragraph{Two ISIs} Here we test the model's generalizability by introducing it to signals where there are only two ISI present, separated by a longer than trained ITI. The objective is to determine if the models truly understand the onset of the ISI and reset their memory. The results in \cref{fig:two-isi} show that both models correctly predict both offsets.

\begin{figure}[H]
    \centering
    \includegraphics[width=0.8\linewidth]{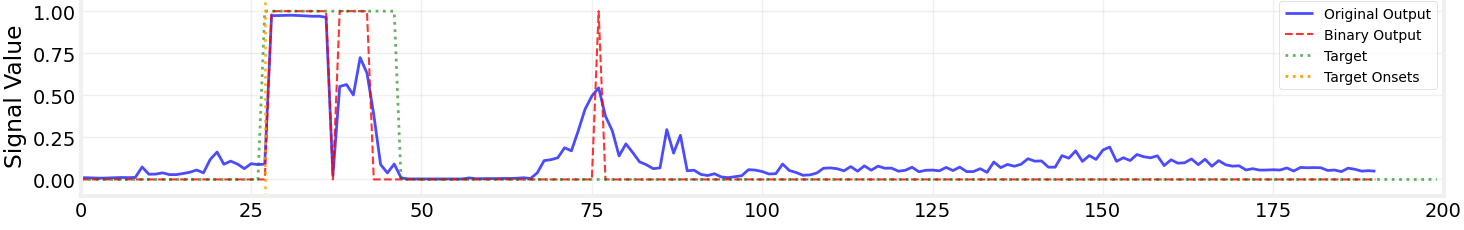}
    \caption{Double ISI outputs for  S4D. We observe  that S4D can anticipate the second ISI.}
    \label{fig:double-isi}
\end{figure}

\begin{figure}[H]
    \centering
    \includegraphics[width=0.8\linewidth]{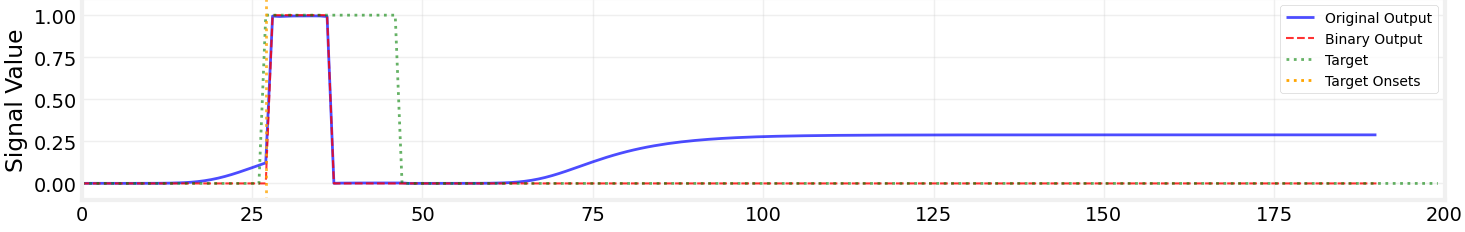}
    \caption{Double ISI outputs for Mamba . Mamba does not anticipate the second ISI.}
    \label{fig:double-isi}
\end{figure}
\paragraph{Double ISI} This tests the model's ability to count when introduced to two consecutive ISIs. The model should turn off at the first offset and then turn on again before the end of the second offset. \cref{fig:double-isi} shows that S4D predicts the initial offset accurately and turns on afterwards. Mamba only predicts the offset correctly, but fails to predict the initial anticipated offset.

\section{Parity Task}\label{app:exp-parity}

Here, we give the details of model sizes for the parity task. For each model, the embedding size and hidden state size are as below
\begin{table}[h]
    \centering
    \begin{tabular}{c|ccc}
    \toprule
    Model & number of layers & embedding size &hidden size (for RNN) /state size (for SSM) \\
    \midrule
    RNN & 1 & 2 & 8    \\
    S4D & 2 & 8 & 16  \\
    Mamba & 2 & 8 & 16   \\
    Mamba + S4D & 2 & 8 &  16  \\
   
    \bottomrule
    \end{tabular}
    \caption{Performance of various models on the parity task.}
    \label{tab:parity}
\end{table}

\section{Product of PSD Matrices is non-Negative}\label{app:psd-product}
The proof is based on the following definition and the two following lemmas.
\begin{definition}\label{def:pds}
    (Positive Semi-Definite Matrix)
    An $n\times n$ Matrix $A$ is positive semi-definite (PSD) iff $\forall x\in \mathbb{R^n}: x^TAx\geq0$.
\end{definition}
\begin{lemma}\label{lem:psd_mult}
For two positive semi-definite (PSD) matrices $A$ and $B$, which are also real and symmetric, the product $AB$ has only non-negative values.
\end{lemma} 
\begin{proof}
The proof is based on the following lemmas.
\begin{lemma}\label{lem:eigenvalues}
   For two square matrices $S$ and $T$, $ST$ has the same non-zero eigenvalues as $TS$.
\begin{proof}
    Let $\lambda$ be an eigenvalue of $ST$ with the corresponding eigenvector $v$: $STv = \lambda v$. Multiplying both sides by $T$, we have $TSTv = \lambda Tv$. Now, redefining $v' = Tv$, the preceding relation can be rewritten as $TSv' = \lambda v'$, which means that $v'$ is  an eigenvector of $TS$ with the same eigenvalue $\lambda$. Since for every non-zero eigenvector $v$ of $ST$, $v'$ can be defined, the set of non-zero eigenvalues will be the same.
\end{proof}
\end{lemma}
\begin{lemma}\label{lem:psd-square-root}
    (Principal Square Root of positive semi-definite (PSD) Matrix) A symmetric real PSD matrix $A$ has a unique square root $B^2=A$ which is also symmetric and PSD. 
\end{lemma}
From~\cref{lem:psd-square-root} the principal square root of $A$ exists; we call it $\sqrt{A}$. We can write $AB = \sqrt{A}\sqrt{A}B = \sqrt{A}(\sqrt{A}B)$. \cref{lem:eigenvalues} proves that $AB$ has the same eigenvalues as $(\sqrt{A}B)\sqrt{A}$. Therefore, if we show that the eigenvalues of $(\sqrt{A}B)\sqrt{A}$ are non-negative, the lemma is proved.
From the symmetry of the principar square root, we have $(\sqrt{A}B)\sqrt{A} = (\sqrt{A}^TB)\sqrt{A}$ and we can show that it is PSD, because $B$ is so. We use ~\cref{def:pds}. 
\begin{align}
x^T(\sqrt{A}^TB)\sqrt{A}x = (\sqrt{A}x)^TB\sqrt{A}x = v^T Bv \geq 0\,,\quad v = \sqrt{A}x
\end{align}
Now, since $(\sqrt{A}B)\sqrt{A} = (\sqrt{A}^TB)\sqrt{A}$ and $(\sqrt{A}^TB)\sqrt{A}$ is PSD, $$(\sqrt{A}B)\sqrt{A}$$ is also PSD with non-negative eignenvalues. Hence, the proof is complete.
\end{proof}
\end{document}